\documentclass[11pt]{article}

\usepackage{amsmath,amsthm,amssymb,amscd,amstext,amsfonts}
\usepackage{color,verbatim,graphicx,fullpage,url}
\usepackage{xcolor}
\usepackage{algorithm}
\usepackage{algorithmic}
\usepackage{nicefrac}
\usepackage{fullpage,tikz,enumitem}
\usepackage{tikz-cd}
\usepackage{bm}
\usepackage{physics}
\usepackage{todonotes}
\usepackage{mathtools}
\usepackage{thm-restate}

\usepackage[breaklinks=true, linktocpage=true,]{hyperref}
\hypersetup{
    colorlinks = true,
    linkcolor={purple},
    urlcolor={purple},
    citecolor={blue!80!black}
}

\usepackage[nameinlink, noabbrev, capitalize]{cleveref}

\usepackage[utf8]{inputenc}
 
\Crefname{enumi}{Property}{Properties}

\theoremstyle{plain}

\newtheorem{thm}{Theorem}[section]
\newtheorem{theorem}[thm]{Theorem}

\crefname{atheorem}{Theorem}{Theorems}
\Crefname{atheorem}{Theorem}{Theorems}

\newtheorem{corollary}[thm]{Corollary}
\newtheorem{lemma}[thm]{Lemma}

\newtheorem{definition}[thm]{Definition}

\newtheorem{claim}[thm]{Claim}
\newtheorem*{claim*}{Claim}

\theoremstyle{remark}
\newtheorem{remark}[thm]{Remark}

\renewcommand\norm[1]{\left|\!\left|#1\right|\!\right|}

\renewcommand{\abs}[1]{|#1|}

\newcommand{\set}[1]{\{#1\}}
\newcommand{\Set}[1]{\left\{#1\right\}}

\newcommand{\inp}[2]{{\left\langle #1,#2 \right\rangle}}            % inner product

\def\1{\mathbf{1}} 
\def\0{\mathbf{0}}

\DeclareMathOperator{\GHD}{GHD}             % Gap Hamming Distance

\DeclarePairedDelimiter\ceil{\lceil}{\rceil}             % integer part
\DeclareMathOperator{\Ex}{\mathbb{E}}           % expected value 
\DeclareMathOperator{\Prob}{Pr}                        % probability
\def\supp{{\mathrm{supp}}}                               % support 

\newcommand{\N}{\mathbb{N}}

\newcommand{\R}{\mathbb{R}}
\newcommand{\bbS}{\mathbb{S}}

\newcommand{\cI}{\mathcal I}
\newcommand{\cX}{\mathcal X}
\newcommand{\cY}{\mathcal Y}
\newcommand{\cZ}{\mathcal Z}

\newcommand{\bS}{\mathbf S}

\newcommand{\bx}{\bm{x}}
\newcommand{\by}{\bm{y}}

\DeclareMathOperator{\sign}{sgn}

\DeclareMathOperator{\margin}{m} 
\DeclareMathOperator{\signrank}{sign-rank}

\usepackage{ifthen}
\newcommand{\agnorm}[2][]{
	\ifthenelse{\equal{#2}{}}{
		\widetilde{\gamma}_2^{#1}
	}{
		\widetilde{\gamma}_2^{#1}(#2)
	}
}

\DeclareFontFamily{U}{mathx}{}
\DeclareFontShape{U}{mathx}{m}{n}{<-> mathx10}{}
\DeclareSymbolFont{mathx}{U}{mathx}{m}{n}
\DeclareMathAccent{\widecheck}{0}{mathx}{"71}

\newcommand{\cH}{\mathcal{H}}
\newcommand{\cA}{\mathcal{A}}

\newcommand{\cC}{\mathcal{C}}

\newcommand{\E}{\mathbb{E}}

\newcommand{\bP}{{\bm P}}
\newcommand{\bw}{\bm{w}}
\newcommand{\bz}{\bm{z}}

\DeclareMathOperator{\ldim}{Ldim}
\DeclareMathOperator{\VCdim}{VCdim}
\DeclareMathOperator{\LR}{LR}
\newcommand{\tLR}{\widetilde{\LR}}

\newcommand{\gs}{\rho^\mathrm{gs}}

\DeclareMathOperator{\loss}{loss}                                             % loss

\begin{document}

\title{Borsuk-Ulam and replicable learning of large-margin halfspaces}

%\author{Anonymous Author(s)}

\author{
    Ari Blondal\thanks{McGill University, \texttt{\{ari.blondal, hamed.hatami\}@mcgill.ca}. Hamed Hatami is supported by an NSERC grant.} \and Hamed Hatami\footnotemark[1] \and
    Pooya Hatami~\thanks{Ohio State University, \texttt{\{hatami.2, lalov.1, tretiak.2\}@osu.edu}} \and  Chavdar Lalov\footnotemark[2] \and Sivan Tretiak\footnotemark[2]
}

\maketitle

\begin{abstract} 
We prove that the list replicability number of $d$-dimensional $\gamma$-margin half-spaces satisfies \[ \frac{d}{2}+1 \le  \LR(\cH^d_\gamma) \le d, \]
which grows with dimension. This resolves several open problems:

\begin{itemize}
\item Every disambiguation of infinite-dimensional large-margin half-spaces to a total concept class has unbounded Littlestone dimension, answering an open question of Alon, Hanneke, Holzman, and Moran (FOCS '21). 
\item Every disambiguation of the Gap Hamming Distance problem in the large gap regime has unbounded public-coin randomized communication complexity. This answers an open question of Fang, G{\"o}{\"o}s, Harms, and Hatami (STOC '25).
\item There is a separation of $O(1)$ vs $\omega(1)$ between randomized and pseudo-deterministic communication complexity. 
\item The maximum list-replicability number of any \emph{finite} set of points and homogeneous half-spaces in $d$-dimensional Euclidean space is $d$, resolving a problem of Chase, Moran, and Yehudayoff (FOCS '23).
\item There exists a partial concept class with Littlestone dimension $1$ such that all its disambiguations have infinite Littlestone dimension. This resolves a problem of Cheung, H. Hatami, P. Hatami, and Hosseini (ICALP '23).
\end{itemize}
Our lower bound follows from a topological argument based on a local Borsuk-Ulam theorem. For the upper bound, we construct a list-replicable learning rule using the generalization properties of SVMs.
\end{abstract}

\pagebreak 

\tableofcontents

\section{Introduction}
Large-margin half-space classification is a fundamental problem in learning theory. In this setting, data is normalized to lie on the unit sphere $\mathbb{S}^{d-1} \subset \mathbb{R}^d$, and we are guaranteed a \emph{promise} that each point lies at least a fixed \emph{margin} $\gamma\in(0,1)$ from some unknown homogeneous hyperplane. The \emph{learner} is then tasked with classifying points based on the side of the defining hyperplane to which they belong. 
This problem has been extensively studied for both its theoretical and practical significance: it provides a clean geometric model for analyzing more complex learning tasks, and underlies the success of Support Vector Machines (SVMs), which leverage the large-margin assumption to produce accurate classifications in high-dimensional spaces, with applications across domains such as text and image recognition, bioinformatics, and fraud detection.

We study this problem through the lens of replicability, the requirement that an algorithm produce consistent outcomes when repeated under similar conditions and with similar data. In recent years, replicability has become a vibrant research area, and various rigorous formulations of replicability for learning algorithms have been introduced and studied~\cite{BLM20,malliaris2022unstable, chase2023replicabilitystabilitylearning, bun2023stability, karbasi2023replicability, esfandiari2023replicable, Esfandiarietal23, moran2023bayesian, eaton2024replicable, kalavasis2024replicable, kalavasis2023statistical}. Among them, one of the most intriguing is the notion of \emph{global stability}, which was first discovered in connection with differentially private learning and online learning~\cite{BLM20, Alon_22_private_and_online, chase2023replicabilitystabilitylearning}. Subsequent work, however, has shown that its significance extends well beyond these applications. In its equivalent formulation as \emph{list replicability}~\cite{chase2023replicabilitystabilitylearning}, the notion is intrinsically linked to the geometry and topology of the space of realizable distributions of a concept class~\cite{chase2023replicabilitystabilitylearning, chase2023local, BGHH2025stabilitylistreplicabilityagnosticlearners, chornomaz2025spherical, Topological}.  For example, for finite classes, it characterizes the \emph{topological dimension} of this space under its natural simplicial structure~\cite{Topological}.  Through these connections, fundamental results in classical topological dimension theory, such as the Lebesgue covering theorem, translate directly into statements about learnability and replicability.

Our main result states that the \emph{list replicability number} of the large-margin classification problem in $\mathbb{R}^d$ lies between $\tfrac{d}{2}+1$ and $d$. In particular, it diverges as the ambient dimension $d$ grows. This stands in contrast to many common complexity measures for the same task, such as the VC dimension, Littlestone dimension, and randomized communication complexity, which are bounded by a function of the margin $\gamma \in (0,1)$ independent of the ambient dimension $d$.

This divergence has several consequences and resolves a number of open problems from previous works, as discussed in detail in \Cref{sec:Applications}. Beyond its implications within learning theory, it also connects naturally to questions in communication complexity. In particular, our most notable consequence shows that any \emph{disambiguation} of the large-margin classification problem into a total concept class must have a large Littlestone dimension and a large randomized communication complexity. Thus, while the original \emph{partial} problem is ``easy'' under these classical measures, every possible completion of it to a \emph{total} problem is inherently ``hard''. Establishing such lower bounds for disambiguations is typically very challenging, as the initial partial problem is ``easy'' and one has no control over how it is extended to a total one.

A notable consequence of our disambiguation theorem is an $O(1)$ versus $\Omega(\log \log n)$ gap between randomized and pseudo-deterministic communication complexities. Separating these two measures is a well-known open problem in the study of pseudodeterminism~\cite{MR4287109}, and our result provides the first $O(1)$ versus $\omega(1)$ separation. 

\subsection{Preliminaries} 
\label{sec:prelim}

We study the large-margin half-space problem through the formal lens of partial classes, which offers a general framework for analyzing such constrained learning tasks.

A \emph{partial concept class} over an arbitrary domain $\cX$ is a set $\cC \subseteq \Set{\pm 1,\star}^\cX$, where each $c \in \cC$ is called a \textit{partial concept}. The value $c(x) = \star$ indicates that $c$ is \textit{undefined} at $x$, and therefore, both $\pm 1$ are acceptable predictions for the label of $x$. 

\paragraph{PAC learning.}  The standard mathematical framework for analyzing the complexity of a learning task is \emph{probably approximately correct (PAC) learning}.  In PAC learning, the learner is given parameters $\delta,\epsilon>0$ and receives training data consisting of $n=n(\mathcal{C},\delta,\epsilon)$ independent labeled examples drawn from an unknown but fixed distribution $\mu$ over $\mathcal{X} \times \set{\pm 1}$.  We work in the \emph{realizable} setting: for every $n$, a random sample $\bm{S} = ((\bx_i, \by_i))_{i=1}^n \sim \mu^n$ is almost surely realizable by some $c \in \cC$, meaning that $c(\bx_i) = \by_i$ for all $i=1,\ldots,n$.\footnote{Here, and throughout the paper, we use boldface letters to denote random variables and use the notation $(\bm{x},\bm{y})\sim \mu$ to express that $(\bm{x},\bm{y})$ is a random variable distributed according to $\mu$.}
Note that $\mu$ is a distribution over $\cX \times \Set{\pm 1}$, so none of the labels $\by_i$ take the value $\star$. The learner’s task is to use the training data to output, with probability at least $1-\delta$, a \emph{hypothesis} $h:\cX \to \set{\pm 1}$ whose \emph{population loss}
\[\loss_\mu(h) \coloneqq \Pr_{(\bx,\by) \sim \mu}[h(\bx) \neq \by]\]
is at most $\epsilon$. 

The following simple lemma from~\cite{alon2022theory} establishes the connection between realizability and having zero population loss.
 
\begin{lemma}[\cite{alon2022theory}]
Let $\cC \subseteq \Set{\pm 1,\star}^\cX$ be a partial concept class, and let $\mu$ be a distribution on $\cX \times \Set{\pm 1}$. If $ \loss_\mu(\cC)\coloneqq \inf_{c \in \cC}\loss_\mu(c)$ is zero, then $\mu$ is realizable by $\cC$. Conversely, if $\mu$ is realizable and has finite or countable support, then $ \loss_\mu(\cC) = 0 $.
\end{lemma}

The fundamental theorem of PAC learning states that the size of the training set required for PAC learning a \emph{total} concept class depends on a combinatorial parameter known as the VC dimension, which we now define in the more general partial setting.  A (partial) concept class $\cC \subseteq \set{\pm 1,\star}^\cX$ shatters a set $S \subseteq \cX$ if $\set{ c|_S : c \in \cC } = \set{\pm 1}^S$, where $c|_S$ denotes the restriction of $c$ to $S$.  The VC dimension of $\cC$ is defined as 
\[ \VCdim(\cC) \coloneqq \sup \set{ |S|: \ S\subseteq \cX \text{ is shattered by }\cC}. \] 

In~\cite{alon2022theory}, Alon, Hanneke, Holzman, and Moran proved that the fundamental theorem of PAC learning holds for partial concept classes as well.

\paragraph{List replicability.}  Throughout this work, a \emph{learning rule} refers to a (randomized) function $\bm{\cA}$ that maps any sample
$S \in \bigcup_{n=0}^\infty (\cX \times \Set{\pm 1})^n$ 
to a \emph{hypothesis} $\bm{\cA}(S) \in \Set{\pm 1}^{\cX}$. Since our primary focus is sample complexity rather than computational efficiency, we impose no computability constraints on $\bm{\cA}$.

\begin{definition}[List replicability]
\label{def:list}
A learning rule $\bm{\cA}$ is an $(\epsilon,L)$-list replicable learner for $\cC \subseteq \Set{\pm 1,\star}^\cX$ if for every $\delta>0$, there is a sample complexity $n=n(\delta)$ such that the following holds. For every realizable distribution $\mu$ on $\cX \times \Set{\pm 1}$, there exists $h_1,\ldots,h_L \in \Set{\pm 1}^\cX$ such that 
\[
    \loss_\mu(h_i) \le \epsilon ~\forall i
    ~\text{ and } ~
    \Pr_{\bm{S} \sim \mu^n} [\bm{\cA}(\bm{S}) \not\in \Set{h_1,\ldots,h_L}] \le  \delta.
\]
The $\epsilon$-list replicability number of $\cC$ is
\[
    \LR_\epsilon(\cC) \coloneqq \min\{L : \exists (\epsilon,L)\text{-list replicable learner for }\cH\},
\]
with $\LR_\epsilon(\cC)=\infty$ if none exists. The list replicability number of $\cC$ is 
\[
    \LR(\cC) \coloneqq \sup_{\epsilon>0} \LR_\epsilon(\cC).
\]
We say $\cC$ is \emph{list replicable} if $\LR(\cC)<\infty$.
\end{definition}

\Cref{def:list} provides a strong notion of replicability as the learner's output is typically chosen from a small list $\Set{h_1,\ldots,h_L}$, and all these hypotheses have small population loss.

\begin{remark}
    Some readers might be familiar with an equivalent form of list replicability known as global stability. Its definition is not needed for the main results of this paper, so we include it in \cref{sec:global_stability}, along with the definition of the \emph{global stability parameter} $\rho(\cC)$ of a concept class $\cC$. This parameter is analogous to the list replicability number, and in fact \cite{chase2023replicabilitystabilitylearning} proved that the two quantities are related by the equation $\rho(\cC)=1/\LR(\cC)$ for total concept classes. It is easy to check that the proof extends to the partial setting (see \cref{sec:List_eq_global}). Hence, qualitative results about list replicability also hold for global stability, and quantitative results hold after the appropriate reciprocal modification.
\end{remark}

\paragraph{Online learning and Littlestone dimension.} In \emph{online learning}, a learner receives data points sequentially from an adversary and must predict each label before seeing the correct answer. The goal is to minimize the total number of mistakes. The optimal mistake bound is captured by the Littlestone dimension, a refinement of the VC dimension defined via mistake trees. 

A mistake tree of depth $d$ over domain $\cX$ is a complete binary tree whose internal nodes are labeled by points $x \in \cX$ and edges by bits $b \in \set{\pm 1}$ ($-1$ for left, $+1$ for right). Following a path from the root to a leaf thus yields a sequence $(x_1,b_1),\dots,(x_d,b_d)$, where each $x_i$ is the node label at level $i$ and $b_i$ records whether the path goes left or right.

A concept class $\cC \subseteq \set{\pm 1,\star}^\cX$ shatters such a tree if for every root-to-leaf path there exists $c \in \cC$ with $c(x_i)=b_i$ for all $i$. The Littlestone dimension $\ldim(\cC)$ is the largest $d$ for which some depth-$d$ mistake tree is shattered, or $\infty$ if no such $d$ exists.

It always holds that $\VCdim(\cC) \leq \ldim(\cC)$, since any set $S = \Set{x_1, \ldots, x_d}$ shattered by $\cC$ gives rise to a mistake tree of depth $d$, where all nodes at level $i$ are labeled with $x_i$. This tree is shattered by $\cC$.  

Littlestone proved that a total concept class $\cC$ is online learnable if and only if $\ldim(\cC) < \infty$. This result was later extended to partial concept classes by~\cite{alon2022theory}

\paragraph{Large-margin half-spaces.} In the large-margin setting, the domain is $\mathbb{S}^{d-1}$ and every homogeneous half-space defines a partial concept that assigns $c(x) = \star$ if $x$ lies within distance $\gamma$ of the defining hyperplane of $h$. Otherwise, it classifies $x$ as $\pm 1$ depending on whether it belongs to the half-space. More formally, each concept $c_w:\bbS^{d-1} \to \set{\pm 1,\star}$ is specified by a unit vector $w \in \bbS^{d-1}$ and given by
\begin{equation}
\label{eq:def_large_margin}
        c_w(x) \coloneqq
        \begin{cases}
            \sign(\inp{w}{x}) & \text{if } \abs{\inp{w}{x}} \ge \gamma\\
            \star & \text{otherwise}
        \end{cases}.
\end{equation} 
We denote the partial concept class of all such $c_w$ by $\cH^d_{\gamma}$.

For the standard half-space classification problem $\cH^d$ without any margin assumption (that is, each $x$ is labeled by $\sign(\inp{w}{x})$ whenever $\inp{w}{x}\neq 0$ and by $\star$ otherwise), we have $\VCdim(\cH^d)=d$.  Moreover, $\ldim(\cH^d)=\infty$, except in the trivial case of $d = 1$. In particular, this class is not online learnable, even in $\mathbb{R}^2$.   

However, under the large-margin assumption $\gamma>0$, the classic mistake-bound analysis of the Perceptron algorithm~\cite{MR10388,rosenblatt1958perceptron} (see also~\cite[Theorem 9.1]{10.5555/2621980}) shows the following upper bound on the Littlestone and VC dimensions:
\begin{equation} \label{eq:Ldim_of_margin}
\VCdim(\cH_\gamma^d) \le \ldim(\cH_\gamma^d) \le \gamma^{-2}.
\end{equation}
Crucially, these bounds are independent of $d$, which explains the efficient PAC and online learnability of $\cH_\gamma^d$ in arbitrarily high-dimensional spaces.

\paragraph{Gap Hamming problem.} The discrete analogue of large-margin half-spaces is the well-studied \emph{Gap Hamming Distance} (GHD) problem, a central problem in communication complexity. For $n \in \mathbb{N}$ and $\gamma \in (0,1)$, the $n$-bit $\operatorname{GHD}_\gamma$ problem, denoted $\operatorname{GHD}_\gamma^n$, is the partial function on inputs $x, y \in \{\pm 1\}^n$ defined by
    \[
        \operatorname{GHD}_{\gamma}^n (x,y)
        \coloneqq
        \begin{cases}
            \sign(\inp{x}{y}) & \text{if $|\inp{x}{y}|  > \gamma n$}\\
            \star & \text{otherwise}
        \end{cases}.
    \]
As a communication problem, Alice receives $x$, Bob receives $y$, and under the promise $|\inp{x}{y}| \ge \gamma n$, they must compute $\GHD_\gamma^n(x,y)$ with minimal communication. For fixed $\gamma$, the public-coin randomized communication complexity of $\GHD_\gamma^n$ is $O_\gamma(1)$: using shared randomness, the players sample a subset $S \subseteq [n]$ of size $O_\gamma(1)$ and estimate $\inp{x}{y}$ by $\tfrac{n}{|S|}\sum_{i \in S} x_i y_i$, which requires only $2|S|$ communicated bits.

\section{Main theorem}
Our main theorem determines the list replicability number of $\cH_\gamma^d$ up to a factor of two, showing in particular that it grows unboundedly with the dimension $d$. 

\begin{theorem}[Main theorem]
\label{thm:main}
For any fixed margin $\gamma \in (0,1)$, dimension $d>1$, and accuracy parameter $\epsilon \in (0,1/2)$,
\[
\frac{d}{2}+1  \leq  \LR_\epsilon(\cH_\gamma^d)  \leq  d.
\]
Hence, $\tfrac{d}{2}+1 \leq \LR(\cH_\gamma^d) \leq d$.
\end{theorem}

The lower bound in \Cref{thm:main} relies on a topological argument involving covers of the sphere by antipodal-free open sets.  In particular, we apply the local Borsuk-Ulam theorem of~\cite{chase2023local}, which states that in such a cover, there is a point that belongs to at least $\frac{d}{2}+1$ sets. Alternatively, one could use Ky Fan's classical theorem~\cite{MR51506}, but this would yield the slightly weaker lower bound of $\frac{d}{2}$.

For the upper bound, we construct a learning rule that uses the generalization properties of hard-SVM combined with a list-replicable rounding scheme using a fine net in general position. 

\subsection{Applications}\label{sec:Applications}

\paragraph{Separation.} In addition to list replicability, differentially private learnability and shared-randomness replicability\footnote{In prior works, shared-randomness replicability is referred to simply as replicability. Since we work with multiple replicability notions, we adopt this terminology to emphasize that different executions of the algorithm use the same random seed.} are two other well-studied notions of stability in learning theory (see~\cite{moran2023bayesian} for an overview). For completeness, formal definitions of these concepts appear in the appendix, though we do not rely on them directly in this paper.

Recent advances in learning theory~\cite{alon2019private,BLM20,Alon_22_private_and_online,chase2023replicabilitystabilitylearning,ILPS22}, sparked by the influential works on differential privacy in PAC learning, have established that for total concept classes, all of these notions coincide and are characterized by the finiteness of the Littlestone dimension. Specifically, for every total concept class $\cC \subseteq \set{\pm 1}^\cX$, the following are equivalent:
\begin{itemize}
\item $\ldim(\cC)<\infty$;
\item $\cC$ is list replicable;
\item $\cC$ is shared-randomness replicable;
\item $\cC$ is approximately differentially private (DP)-learnable.
\end{itemize} 
This naturally raises the question of whether these equivalences extend to partial concept classes. The case of large-margin half-spaces has been studied extensively~\cite{Blum2005,le2020efficient,beimel2019private,kaplan2020private,bun2020efficient,bassily2022differentially,bassily2022open,ILPS22,kalavasis2024replicable}, and the following are known:
\begin{itemize}
\item $\ldim(\cH^d_\gamma) < \gamma^{-2}$;
\item $\cH^d_\gamma$ is (pure) DP-learnable with dimension-independent sample complexity;
\item $\cH^d_\gamma$ is shared-randomness replicable with dimension-independent sample complexity.
\end{itemize}

Nevertheless, \Cref{thm:main} shows that despite these strong positive results, the list replicability number of $\cH^d_\gamma$ necessarily grows with $d$. Consequently, we obtain a sharp separation from the total setting: for partial classes, list replicability does \emph{not} follow from bounded Littlestone dimension, replicability, DP-learnability, or even pure DP-learnability.

\begin{corollary}[Separation]
\label{cor:separation}
There exists a partial concept class $\cC$ that is (pure) DP-learnable, shared-randomness replicable, and satisfies $\ldim(\cC)<\infty$, yet it is not list replicable.
\end{corollary}
\begin{proof}
Fix $\gamma \in (0,1)$, and define the class $\cH^\infty_\gamma$ as follows. Each hypothesis in $\cH^\infty_\gamma$ is specified by a unit vector $w$ of arbitrary finite dimension, i.e., $w \in \bigcup_{d \in \mathbb{N}} \bbS^{d-1}$. For $x \in \bigcup_{d \in \mathbb{N}} \bbS^{d-1}$, define
\[
c_w(x) \coloneqq \begin{cases}         \sign(\inp{w}{x}) & \text{if $\dim(x) = \dim(w)$ and $\abs{\inp{w}{x}} \geq \gamma$,}\\
\star & \text{otherwise} \\  
\end{cases}.
\]

By the aforementioned results of \cite{MR10388,rosenblatt1958perceptron,le2020efficient,kalavasis2024replicable}, the class $\cH^\infty_\gamma$ is pure DP-learnable, shared-randomness replicable, and satisfies $\ldim(\cH^\infty_\gamma) < \gamma^{-2}$. However, by \Cref{thm:main}, we have $\LR(\cH^\infty_\gamma) = \infty$. This establishes the claim.
\end{proof}

\paragraph{Disambiguations of large-margin half-spaces.} A \emph{disambiguation} of a partial concept class $\cC \subseteq \Set{\pm 1,\star}^\cX$ is a total concept class $\overline{\cC} \subseteq \Set{\pm 1}^\cX$ such that for every $c \in \cC$ and every \emph{finite} $S \subseteq c^{-1}(\Set{\pm 1})$, there exists an $\bar{c}\in \overline{\cC}$ that is consistent with $c$ on $S$.  Intuitively, this corresponds to resolving each $\star$ with $-1$ or $+1$, although this intuition is not completely rigorous in the infinite case.

Disambiguation cannot decrease the list-replicability number. At the same time, it converts a partial class into a total class, where the highly nontrivial results of \cite{BLM20,Alon_22_private_and_online,ghazi2021sample} show that list replicability is bounded in terms of the Littlestone dimension. This principle underlies our results on disambiguations of the large-margin half-space problem and its discrete analogue $\operatorname{GHD}_\gamma$.  The following theorem resolves an open question of \cite[Question 4]{alon2022theory} who asked whether every disambiguation of $\cH_\gamma^d$ satisfies $\ldim=\omega(1)$.

\begin{theorem}\label{thm:disambig}
For every $d\in \N$, every disambiguation $\overline{\cH}$ of $\cH_\gamma^d$ satisfies $\ldim(\overline{\cH})=\Omega(\sqrt{\log d})$. 
\end{theorem}
\begin{proof}
As was noted in \cite{chase2023replicabilitystabilitylearning}, it is implicitly~\footnote{See \cite[Lemma 5.5]{ghazi2021sample} and the definitions of $k_t$ and $k'$, where $k'$ depends on $n_0$, $d_L$, $\alpha_\Delta$, $C_0$, and $\eta^2$ as specified in \cite[Algorithm 2]{ghazi2021sample}.} proved in \cite{ghazi2021sample} that
for every class $\overline{\cH}$, 
\[
\LR_\epsilon(\overline{\cH})\leq 2^{O_\epsilon\left(\ldim(\overline{\cH})^2\right)}.
\]
Recall that by \Cref{thm:main}, for every $\epsilon \in (0,1/2)$, 
$\LR_\epsilon(\overline{\cH}) \ge \frac{d}{2}+1$.
Combining the two inequalities for a fixed $\epsilon\in (0,1/2)$ concludes the claim. 
\end{proof}

\cite{alon2022theory} used a sophisticated construction, building on G\"o\"os' breakthrough refutation of the Alon–Saks–Seymour conjecture~\cite{MR3473357}, to exhibit partial concept classes with $\VCdim=2$ whose disambiguations satisfy $\ldim=\omega(1)$. Subsequently,~\cite{cheung2023online} employed a similar approach to construct partial classes with $\ldim=2$ whose disambiguations again satisfy $\ldim=\omega(1)$. \Cref{thm:disambig} provides a much more natural example of a class exhibiting this phenomenon.

Moreover,~\cite[Question~4.1]{cheung2023online} asked whether such a separation can already occur for partial classes of Littlestone dimension~$1$. The following corollary of \Cref{thm:disambig} answers this question in the affirmative.

\begin{corollary}
Let $\gamma \in (1/\sqrt{2},1)$. Then $\ldim(\cH_\gamma^{d})=1$, while any disambiguation of $\cH_\gamma^{d}$ has  Littlestone dimension $\Omega(\sqrt{\log d})$.  
\end{corollary}
\begin{proof}
We have $1\leq \ldim(\cH_\gamma^{d})\leq 1/\gamma^2<2$. An application of \Cref{thm:disambig} completes the proof.
\end{proof}

\paragraph{Disambiguations of gap Hamming distance.} 
In complexity theory, separations between complexity measures are often easier to demonstrate for partial functions, and disambiguations of important partial functions are studied as a way to extend such results to the total setting. For fixed $\gamma \in (0,1)$ the partial $\operatorname{GHD}_\gamma$ is known to separate constant cost randomized communication complexity from several other important complexity measures \cite{hambardzumyan2023dimension,chattopadhyay2019equality,hatami2023borsuk,song2014space}. Motivated by this, researchers have asked whether $\operatorname{GHD}_\gamma$ admits a disambiguation with constant cost randomized communication complexity \cite{fang2024constant}. Our next theorem gives a negative answer to this question.

\begin{theorem}\label{thm:GHD}
Let $\gamma \in (0,1)$ be a margin parameter. Every family of disambiguations $\{M_n\}_{n=1}^{\infty}$ of the Gap Hamming Distance matrices $\{\mathrm{GHD}_{\gamma}^n\}_{n=1}^\infty$ satisfies
\begin{equation}\label{eq:ghd-ldim}
\ldim(M_n)=\Omega(\sqrt{\log n}),
\end{equation}
and has public-coin randomized communication complexity $\Omega(\log\log n)$.
\end{theorem}

To lower bound the Littlestone dimension, we use an embedding due to \cite{hatami2023borsuk} that allows us to disambiguate $\cH_\gamma^d$ using a disambiguation of the Gap Hamming Distance problem in dimension $O(d)$. The key here is that the embedding will maintain the lower bound on Littlestone dimension. \cref{eq:ghd-ldim} then follows from \cref{thm:disambig}. The bound on the communication complexity then follows as a corollary, using the known relationship between Littlestone dimension, \emph{margin}, \emph{distributional discrepancy}, and public-coin randomized communication complexity. See \Cref{sec:GHD} for the proof. 

\paragraph{Pseudo-determinism vs randomness in communication.}

A pseudo-deterministic algorithm is a randomized algorithm that, when executed multiple times on the same input, produces the same output with high probability. This notion was introduced by Gat and Goldwasser~\cite{GatG11} and has since been extensively investigated across a variety of computational models, including learning algorithms, communication protocols, decision tree algorithms, sequential and parallel algorithms, average-case and approximation algorithms, interactive proofs, low-space algorithms, and streaming algorithms (see~\cite{MR4048182,MR4287109} and the references therein). A central question in this line of research is to understand to what extent pseudo-determinism can be separated from general randomized computation.

In communication complexity, a search problem is specified by a relation $\mathcal{R} \subseteq \cX \times \cY \times \cZ$, where Alice and Bob receive $x \in \cX$ and $y \in \cY$ respectively, and must output $z \in \cZ$ such that $(x,y,z) \in \mathcal{R}$ while minimizing communication. In the public-coin randomized model, the players have access to a shared random string $\bm{r}$, and a protocol $\pi$ must satisfy
\[\Pr_{\bm{r}}[(x,y,\pi(\bm{r},x,y)) \in \mathcal{R}] \ge \frac{2}{3} \ \ \forall (x,y) \in \cX \times \cY. \] 
Such a protocol is called pseudo-deterministic if there exists a function $f:\cX \times \cY \to \cZ$ (with $(x,y,f(x,y)) \in \mathcal{R}$ for all $(x,y)$) such that 
\[\Pr_{\bm{r}}[\pi(\bm{r},x,y)=f(x,y)] \ge \frac{2}{3} \ \ \forall (x,y) \in \cX \times \cY. \] 
A well-known candidate for separating pseudo-determinism from randomized communication is the \emph{approximate Hamming distance problem} 
\[(x,y,t) \in \operatorname{AHD}_n \ \Longleftrightarrow \ |d_H(x,y)-t| < \tfrac{n}{3},\] 
where $d_H(x,y)$ is the Hamming distance between $x,y \in \set{\pm 1}^n$. In the public-coin randomized model, Alice and Bob can solve $\operatorname{AHD}_n$ with only $O(1)$ bits of communication by sampling $O(1)$ coordinates uniformly at random and exchanging the corresponding entries to estimate $d_H(x,y)$.

By contrast, it is widely believed that the pseudo-deterministic communication complexity of $\operatorname{AHD}_n$ is large. The following theorem establishes the first super-constant lower bound on this problem, thereby yielding the first $O(1)$ versus $\omega(1)$ separation between randomized and pseudo-deterministic communication complexities.

\begin{theorem}
For any $\epsilon<\frac{1}{2}$, the pseudo-deterministic communication complexity of $\operatorname{AHD}_n$ is $\Omega(\log \log (n))$. 
\end{theorem}
\begin{proof}
Suppose there is a pseudo-deterministic protocol for $\operatorname{AHD}_n$ of cost $k$, with corresponding function $f:\set{\pm 1}^n \times \set{\pm 1}^n \to \set{0,1,\ldots,n}$. Define 
\[
F(x,y) \coloneqq
\begin{cases}
1 & f(x,y) \ge \frac{n}{2}\\  
-1 & f(x,y) < \frac{n}{2}
\end{cases},
\]
and note that the public-coin randomized communication complexity of $F$ is at most $k$.

Since $f(x,y)$ is always within $n/3$ of the true Hamming distance $d_H(x,y)$, the function $F$ disambiguates the Gap Hamming Distance problem $\GHD_{0.1}^n$, and by \Cref{thm:disambig}, its randomized communication complexity is $\Omega(\log\log n)$.  
\end{proof}

\paragraph{List replicability of finite hyperplane arrangements.} 
Define the \emph{finitary list replicability number} of a concept class $\cC \subseteq \{\pm 1, \star\}^\cX$ as 
\[
\tLR(\cC)\coloneqq \sup_{\text{finite } S\subseteq \cX} \LR(\cC\vert_S).
\]
As an example, consider the class $\cH^2$ of homogeneous halfspaces in $\mathbb{R}^2$. While $\LR(\cH^2)=\infty$, Chase, Moran, and Yehudayoff~\cite[Theorems 5 and 13]{chase2023replicabilitystabilitylearning} proved that for any finite set of points $S \subseteq \mathbb{S}^1$, we have $\LR(\cH^2_0\vert_S) \le 2$. This establishes the following striking gap:
\[
\LR(\cH^2)=\infty \qquad \text{while} \qquad \tLR(\cH^2)=2. 
\] 
They further asked whether a similar bound holds for $\tLR(\cH^3)$ and, more generally, in higher dimensions. The next theorem resolves their question affirmatively.

\begin{theorem}
\label{thm:finitary}
For every dimension $d >1$, we have $\tLR(\cH^d)= d$. 
\end{theorem}
\begin{proof}
The upper bound is an easy consequence of the upper bound of our main theorem (\Cref{thm:main}). Indeed, any finite set of points $S\subset \bbS^{d-1}$ and hypotheses $H\subset \cH^d\vert_S$ defined by unit vectors $W\subset \bbS^{d-1}$ is a sub-concept class of $\cH^d_\gamma$, where $\gamma \coloneqq \min_{x\in S, w\in W} |\inp{w}{x}|$. The claim now follows from our upper bound from \Cref{thm:main}.

For the lower bound, we use a result of Chase, Moran, and Yehudayoff~\cite[Theorem 3]{chase2023replicabilitystabilitylearning} stating that for every concept class $\cC$, 
\[
\LR(\cC)\geq \VCdim(\cC).
\]
The result follows as $\VCdim(\cH^d)=d$, and hence, $\cH^d$ has a finite subclass of VC dimension $d$. 
\end{proof}

\Cref{thm:finitary} reveals a connection between list replicability and one of the most fundamental parameters in learning theory called sign-rank. Geometrically, sign-rank is the smallest dimension in which the matrix is realized as points and homogeneous half-spaces.

\begin{definition}[Sign-rank]
\label{def:signrk}
The sign-rank of a partial class $\cC \subseteq \Set{\pm,\star}^\cX$, denoted by $\signrank(\cC)$, is the smallest $d$ such that there exist vectors $u_c, v_x\in \R^d$ for all pairs $c \in \cC, x \in \cX$ such that $c(x)=\sign(\inp{u_c}{v_x})$ whenever $c(x) \neq \star$.
\end{definition} 

Combining \Cref{thm:finitary} with the VC-dimension lower bound of~\cite[Theorem 3]{chase2023replicabilitystabilitylearning} yields the following general bounds on the finitary list replicability number.

\begin{corollary}
\label{cor:finitary}
For every partial class $\cC \subseteq \Set{\pm,\star}^\cX$, we have 
\[\VCdim(\cC) \le  \tLR(\cC) \le \signrank(\cC). \] 
\end{corollary}

\subsection{Concluding remarks and open problems}   

For \emph{total} concept classes, list replicability, shared randomness replicability, and (approximate) DP-learnability are now known to coincide through the combinatorial framework of the Littlestone dimension. In contrast, the situation for \emph{partial} classes, as illustrated by the results of this paper, is more intricate and less understood.

The ``DP-learnability to Shared-randomness replicability'' reduction from \cite{bun2023stability} extends to the partial setting. Moreover, \cite{fioravanti2024ramsey} recently showed that for partial classes, DP-learnability implies a finite Littlestone dimension, and \cite[Lemma 8]{kalavasis2023statistical} shows that even in the partial setting, list replicability implies shared-randomness replicability.

\begin{theorem}[\cite{bun2023stability,fioravanti2024ramsey,kalavasis2023statistical}]\label{thm:knownimplications}
Let $\cC \subseteq \Set{\pm 1,\star}^\cX$ be a partial concept class. 
\begin{itemize}
\item If $\cC$ is DP-learnable, then $\cC$ is shared-randomness replicable. 
\item If $\cC$ is DP-learnable, then $\ldim(\cC) < \infty$.
\item If $\cC $ is list replicable, then  $\cC$ is shared-randomness replicable. 
\end{itemize}
\end{theorem}

On the other hand, our main theorem shows that for partial concepts, list replicability does not follow from bounded Littlestone dimension, shared-randomness replicability, DP-learnability, or even pure DP-learnability. To our knowledge, no further relationships among DP-learnability, shared-randomness replicability, Littlestone dimension, and list replicability are currently known for partial concept classes.

Finally, we list some open problems for future research that naturally arise from our work. 

\begin{enumerate}
\item Are \emph{DP-learnability}, \emph{shared-randomness replicability}, and finite \emph{Littlestone dimension} equivalent for partial functions? If not, what are the precise relationships between them?

\item Is there a simple combinatorial notion of dimension that characterizes \emph{list replicability}?
 
\item How tight are the inequalities in \Cref{cor:finitary}, namely,   \[\VCdim(\cC) \le  \tLR(\cC) \le \signrank(\cC)? \] 
Is it possible to upper-bound $\tLR(\cC)$ by a function of $\VCdim(\cC)$?

\item Question  4 in \cite{alon2022theory} also asks if every disambiguation of $\cH_\gamma^d$ satisfies $\VCdim=\omega(1)$. Chornomaz, Moran, and Waknine explored this problem using a topological approach, but the question remains open ~\cite{chornomaz2025spherical}.
 
\end{enumerate}

\section{Proof of \Cref{thm:main}}

\subsection{The lower bound}
We prove the lower bound via a topological argument that utilizes the following local version of the Borsuk-Ulam theorem proved in \cite{chase2023local}.

    \begin{theorem}[Local Borsuk-Ulam~\cite{chase2023local}]\label{thm:local-borsukulam}
        Let $d\geq 2$ be an integer. If $\mathcal{F}$ is a finite antipodal-free open cover of the sphere $\bbS^{d-1}$, then there exists some $w\in \bbS^{d-1}$ contained in at least $ \ceil{\frac{d}{2}+1}$ member sets of  $\mathcal{F}$.
    \end{theorem}

Fix any margin $\gamma\in(0,1)$, dimension $d \geq 2$ and $\epsilon \in (0,1/2)$, and suppose that $\bm{\cA}$ is an $(\epsilon,L)$-list replicable learning rule for $\cH_{\gamma}^d$. We prove that $L\geq \frac{d}{2}+1$.

By the definition of list replicability, for any $\delta>0$, there is an integer $n$ so that for any realizable distribution $\mu$, there exists a list of hypotheses $\{h_1,\dots,h_L\}$ with
    \[
        \Prob_{\bS \sim \mu^n}[\bm{\cA}(\bS) \in  \{h_1,\dots,h_L\}] \geq 1-\delta  \text{ and } 
        \loss_\mu(h_i) \leq \epsilon \text{ for all }i\in[L].
    \]
Now pick any $\alpha>0$ and $\epsilon'\in(\epsilon,1/2)$. By taking $\delta$ sufficiently small, for any distribution $\mu$, we can choose a hypothesis $h_{\mu}\in\{h_1,\dots,h_L\}$ such that    \begin{equation}\label{condition open set}
        \Pr_{\bS \sim \mu^n}[\bm{\cA}(\bS) = h_{\mu}] > \frac{1}{L+\alpha} \text{ and } 
        \loss_\mu(h_{\mu}) < \epsilon'.
    \end{equation}

We will focus on a certain collection of realizable distributions $\mu$ on $\bbS^{d-1} \times \{\pm 1\}$. For any $w \in \bbS^{d-1}$, take $\mu_w$ to be the uniform distribution on the set $\{ (x,c_w(x)) \ | \ x \in \supp(c_w) \}$. These distributions are, by definition, realizable. Hence, for each $\mu_w$, we can choose some particular hypothesis $h_{\mu_w}$ that satisfies the conditions in \eqref{condition open set}. Collect these hypotheses in a set $T$, that is to say
    \[
        T \coloneqq \{h_{\mu_w} \ | \ w \in \bbS^{d-1}\}.
    \]
For each $h \in T$, define the set $U_h \subset \bbS^{d-1}$ as 
    \[
        U_h \coloneqq \left\{
        w \in \bbS^{d-1} \ | \ 
         \Pr_{\bS \sim \mu_{w}^n}[\bm{\cA}(\bS) = h] > \frac{1}{L+\alpha} \text{ and } 
        \loss_{\mu_w}(h) < \epsilon'
        \right\}.
    \]
\begin{claim} 
    The family $\{U_h\}_{h\in T}$ forms an antipodal-free open cover of $\bbS^{d-1}$.
\end{claim}
    
    \begin{proof}
        The fact that any set $U_h$ is antipodal-free follows from the accuracy constraint $\loss_{\mu_w}(h) < \epsilon'$. Indeed, for any $w \in \bbS^{d-1}$, the concepts $c_w$ and $c_{-w}$ have identical support, on which they disagree at every point. Thus the population loss of any hypothesis $h$ satisfies the equation
        \[
            \loss_{\mu_w}(h) + \loss_{\mu_{-w}}(h) = 1.
        \]
        For any $w \in U_h$, we have that $\loss_{\mu_w}(h) < \epsilon' < 1/2$, whereby $w$ and $-w$ cannot both be in $U_h$.

        Next, each set $U_h$ is open because both $\Pr_{\bm{S} \sim \mu_{w}^n}[\bm{\cA}(\bm{S}) = h]$ and $\loss_{\mu_{w}}(h)$ are continuous in $w$. Lastly, the family $\{U_h\}_{h\in T}$ covers $\bbS^{d-1}$ because, for any $w \in \bbS^{d-1}$, the set $U_{h_{\mu_w}}$ contains $w$ by construction.
    \end{proof}

Now note that the antipodal-free open cover $\{U_h\}_{h\in T}$ admits a finite subcover by the compactness of the unit sphere. Applying \cref{thm:local-borsukulam} to such a finite subcover guarantees that some $w \in \bbS^{d-1}$ is contained in at least $t \coloneqq \ceil{ \frac{d}{2}+1}$ sets $U_{h_1}, U_{h_2},\dots,U_{h_t}$. Unpacking definitions reveals that the distribution $\mu_{w}$ has the property
    \[
        \Pr_{\bS \sim \mu_{w}^n}[\bm{\cA}(\bS) = h_i] > \frac{1}{L+\alpha}
    \]
for $t$ distinct hypotheses $h_i \in T$. Because these $h_i$ are distinct, the events $[\bm{\cA}(\bS) = h_i]$ are disjoint, and therefore
    \[
        1 \geq \Prob_{\bS \sim \mu_{w}^n} \bigcup_{i=1}^t [\bm{\cA}(\bS) = h_i]
        = \sum_{i=1}^t \Prob_{\bS \sim \mu_{w}^n} [\bm{\cA}(\bS) = h_i]
        > \frac{t}{L+\alpha}.
    \]
It follows that $L+\alpha> t = \ceil{\frac{d}{2}+1}$ for any $\alpha>0$, which implies the desired lower bound $L \geq  \ceil{\frac{d}{2}+1}$.

\subsection{The upper bound}

To prove the upper bound, we design a list replicable learning algorithm $\bm{\cA}$ that learns $\cH_{\gamma}^d$ with list size $d$ independent of $\epsilon>0$. Given $w \in \bbS^{d-1}$, let $\overline{c}_w:\bbS^{d-1} \to \set{\pm 1}$ denote the total concept corresponding to the closed half-space defined by $w$.

\begin{equation*}
        \overline{c}_w(x) \coloneqq
        \begin{cases}
            1 & \text{if } \inp{w}{x}  \ge 0\\
            -1 & \text{if } \inp{w}{x}  < 0
        \end{cases}.
\end{equation*}

 Fundamentally, as in \cite{kalavasis2024replicable}, we estimate a large-margin linear separator using the average of many runs of an SVM maximum margin separator. Then, we use a rounding scheme based on a uniform triangulation of the $\ell_1$ sphere, with the guarantee that with high probability, our learning rule will choose one of at most $d$ separators.

Consider a training sample $(x_1, y_1), \dots, (x_n, y_n) \in \R^d \times \set{\pm 1}$. The (homogeneous) \emph{hard-SVM} is an optimization problem that returns a homogeneous half-space that classifies the training sample correctly while maximizing the margin $\gamma$. More formally, it is the following optimization problem over the variables $\gamma \in \mathbb{R}$ and $w \in \bbS^{d-1}$: 
\[ 
\begin{array}{l rll}
\max  & \gamma  & &\\ 
\text{s.t. } & y_i \inp{x_i}{w}  &\ge  \gamma& \text{for } i=1,\ldots,n \\
& \gamma & \ge 0 & \\
& w & \in \bbS^{d-1} &\\ 
\end{array}
\] 
One can use semi-definite programming to solve this optimization problem efficiently---to check whether it is feasible and, if so, to find the maximizing $w$.

\begin{definition}[$\gamma$-Separator]
        Let $S \subseteq \bbS^{d-1} \times \Set{\pm 1}$. We call $w \in \bbS^{d-1}$ a $\gamma$-separator for $S$ if
        \[
            y  \inp{x}{w}  \geq \gamma \ \ \text{ for all } (x,y) \in S.
        \]
Furthermore, for any distribution $\mu$ over $\bbS^{d-1} \times \{\pm 1\}$, we call $w$ a $(\gamma,\epsilon)$-separator for $\mu$  if 
        \[
            \Pr_{(x, y) \sim \mu}[y \inp{x}{w}< \gamma] \leq \epsilon.
        \]
\end{definition}

When learning $\cH_{\gamma}^d$ in the realizable setting, for any sample set $S$ drawn from a realizable distribution $\mu$, there is some $w$ that $\gamma$-separates $S$. Therefore, the hard-SVM will be feasible and return a vector $w$ that $\gamma$-separates $S$.

The following theorem, due to \cite{TaylorBartlett98}, says that if we take a sufficiently large sample $\bm{S}$ and compute a good separator $w$ for it using hard-SVM, then with high probability, $w$ will also be a good separator for $\mu$. 

\begin{theorem}[{SVM generalization bound~\cite[Theorem 3.5]{TaylorBartlett98}}]\label{thm_svm}
   For all $\epsilon, \delta > 0$, there exists $n \coloneqq n(\epsilon, \delta)$ such that the following holds.  Let $\mu$ be any distribution over $\bbS^{d-1} \times \Set{\pm 1}$.
    \[
        \Pr_{\bm{S} \sim \mu^n}
        \left[\text{Every $w\in \bbS^{d-1}$ that $\gamma$-separates $\bm{S}$ also  $\left(\frac{\gamma}{2},\epsilon\right)$-separates $\mu$}\right] \geq 1 - \delta.
   \] 
\end{theorem}
\begin{remark} 
To prove~\Cref{thm_svm}, one can apply \cite[Theorem 3.5]{TaylorBartlett98} to show that, with probability at least $1-\delta$, both  $h_1(x) \coloneqq \sign (\inp{x}{w}+\frac{\gamma}{2})$ and $h_2(x) \coloneqq \sign (\inp{x}{w}-\frac{\gamma}{2})$ have loss at most $\frac{\epsilon}{2}$, in which case $w$ is a  $\left(\frac{\gamma}{2},\epsilon\right)$-separator for $\mu$. 
\end{remark}

Regarding optimal bounds on $n(\epsilon, \delta)$ in \Cref{thm_svm}, we refer the reader to~\cite{gronlund2020SVM,kalavasis2024replicable}.

First, we prove a simple concentration result for sums of i.i.d. random vectors to show that the outputs of multiple runs of hard-SVM on independent samples are typically concentrated around their mean. 

\begin{lemma}\label{thm_vec_bernstein} 
        Let $\bx_1, \dots, \bx_k \in \R^d$ be i.i.d random variables with mean $\mu$ and $\norm{\bx_i - \mu}_\infty \leq C$. Let $\bm{Z} = \frac{1}{k} \sum_{i=1}^k \bx_i$. For all $t>0$, 
        \[\Pr[\norm{\bm{Z}-\mu}_1 \ge t] \le 2d e^{\frac{-k t^2}{2d^2 C^2}}. \] 
    \end{lemma}    
    \begin{proof}
    We apply Hoeffding's inequality~\footnote{Let $c\in \R$ and let $\bm{x}_1,\ldots, \bm{x}_n$ be independent random variables with $\bm{x}_i\in [-c,c]$ and $\E[\bm{x}_i]=0$. For any $t>0$,
\[
\Pr \left[\left|\sum_{i=1}^n \bm{x}_i\right|\geq t \right]\leq 2e^{-\frac{t^2}{2c}}. 
\]} to each coordinate and take the union bound. By Hoeffding's inequality, for every $j \in [d]$, we have 
    \[\Pr\left[|\bm{Z}_j-\mu_j| \ge \frac{t}{d}\right] \le 2 e^{\frac{-k t^2}{ 2d^2 C^2}}. \]
    Therefore, by the union bound, 
    \[\Pr\left[\norm{\bm{Z}-\mu}_1 \ge t \right] \le 2d e^{\frac{-k t^2}{ 2d^2 C^2}}. \qedhere \]
    \end{proof}

We will use a rounding scheme that ensures any small neighbourhood on $\bbS^{d-1}$ is rounded to at most $d$ points.

\begin{lemma}\label{lemma:good_rounding_scheme}
        For every $\alpha > 0$, there is a $\beta(d)>0$ and a rounding scheme
        \[
            \mathrm{round}_{\alpha} : \bbS^{d-1} \to \bbS^{d-1}
        \]
        such that for all $x \in \bbS^{d-1}$,
        \begin{enumerate}
            \item $\norm{\mathrm{round}_{\alpha}(x) - x }_2 < \alpha$, and
            \item The set $R_x \coloneqq  \Set{
                \mathrm{round}_{\alpha}(y)  \ | \  y \in \bbS^{d-1} \text{ and } \norm{x - y}_2 \leq \beta}$ has size at most  $d$. 
        \end{enumerate} 
    \end{lemma}
\begin{proof}
Consider any $\frac{\alpha}{2}$-net $T$ of points in general position on $\bbS^{d-1}$ and define the rounding function as 
\[
\mathrm{round}_{\alpha}(x) \coloneqq \arg\min_{y \in T} \norm{x-y}_2.
\]
Property 1 follows directly from the definition of $\mathrm{round}_\alpha$, so it remains to prove 2.  

If $|T| \le d$, both conditions are satisfied. Thus, assume $|T|>d$.  
We will use the fact that for any set of $d+1$ distinct points $x_1,\ldots, x_{d+1}\in T$, the origin is the only point equidistant from all of them. To see this, suppose there exists a point $y\in \R^d$ that is equidistant from each $x_i$, meaning there exists some $r$ such that 
\[
r^2=\norm{x_i-y}_2^2= 1+\norm{y}_2^2 - 2\inp{x_i}{y}.
\]
Consequently, $y$ is orthogonal to the linearly independent vectors $x_1-x_2$, \ldots, $x_1-x_{d+1}$, and thus $y=\vec{0}$.

Define the map 
$\phi:\bbS^{d-1} \to \mathbb{R}_{\ge 0}$ 
as 
\[\phi(x) \coloneqq \tau(x)-\min_{y \in T} \norm{x-y}_2,\]
where $\tau(x)$ denotes the distance from $x$ to a $(d+1)$-th closest point in $T$. Since no point in $\mathbb{S}^{d-1}$  can be equidistant to more than $d$ points in $T$, we have $\phi(x)>0$ for all $x$. And since $\phi$ is continuous and  $\bbS^{d-1}$ is compact, we have  
\[ \beta' \coloneqq \min_x \phi(x)>0. \]
Taking $\beta \coloneqq \beta'/3$ completes the proof. 
\end{proof}
    
    \begin{proof}[Upper bound of \Cref{thm:main}]

        We need to show that for any margin $\gamma \in (0, 1)$, accuracy parameter $\epsilon \in (0,1/2)$, and dimension $d \geq 1$,  we have $\LR_\epsilon(\cH_{\gamma}^d) \le d$. 
    
        We will construct a list-replicable learner that always outputs a hypothesis of the form $\overline{c}_{w}$ for some $w \in \bbS^{d-1}$.

        Let $k=k(d,\gamma)$ and $n_0=n_0(\epsilon,\delta,k)$ be integers yet to be determined.
        Consider the following learning rule $\bm{\cA}$ that uses the rounding scheme of  \Cref{lemma:good_rounding_scheme}.

\begin{algorithm}[ht]
\caption{\label{alg:A}The learning rule $\bm{\cA}$}
\begin{algorithmic}[1]
  \FOR{$i \gets 1$ to $k$}
    \STATE Sample $\bS_i \sim \mu^{n_0}$.
    \STATE Let $\bw_i \gets \text{hard-SVM}(\bS_i)$.
  \ENDFOR
  \STATE Let $\bw \gets \frac{1}{k} \sum_{i=1}^k \bw_i$ and $\bz \gets \frac{\bw}{\norm{\bw}_2}$.
  \STATE Let $\tilde{\bz} \gets \mathrm{round}_{\gamma / 2}(\bz)$.
  \STATE \textbf{output} the hypothesis $\overline{c}_{\tilde{\bz}}$.
\end{algorithmic}
\end{algorithm}

We first show that the learning rule $\bm{\cA}$ presented in~\Cref{alg:A} is a PAC learner.

        \begin{claim}
        \label{claim:Average_PAC}
            Let $\bm{\cA}$ and $\bw$ be as in~\Cref{alg:A}.  For every $\epsilon, \delta \in (0, 1)$ and $k \in \N$, there exists $n_0 \coloneqq n_0(\epsilon,\delta,k) \in \N$ such that for every distribution $\mu$ realizable by $\cH_{\gamma}^d$, we have  
            \begin{equation}            
            \label{eq:large_w}
              \Pr_{\bS \sim \mu^{k n_0}}\left[\norm{\bw}_2 <\frac{\gamma}{2}\right]\le \frac{\delta}{4} 
            \end{equation}  
            and
            \begin{equation}
            \label{eq:claim_small_loss}
            \Pr_{\bS \sim \mu^{k n_0}}[\loss_\mu(\bm{\cA}(\bS)) \geq \epsilon] \leq \frac{\delta}{4}.
            \end{equation}
        \end{claim} 
        \begin{proof}
            Let $\bw_1, \dots, \bw_k$ be as in~\Cref{alg:A}. Since $\mu$ is realizable by $\cH_{\gamma}^d$, for every $i$, $\bw_i$ is a  $\gamma$-separator for $\bS_i$. Therefore, by \Cref{thm_svm}, if $n_0(\epsilon,\delta,k)$ is sufficiently large,  
            \[
                \Pr_{\substack{\bS_i \sim \mu^{n_0}} } \left[
                    \Pr_{(\bm{x}, \bm{y}) \sim \mu} \left[
                        \bm{y} \inp{\bm{x}}{\bw_i} < \frac{\gamma}{2}    
                    \right] \leq \frac{\epsilon}{k}
                \right] \geq 1 - \frac{\delta}{4k}.
            \]
            Thus, by the union bound,  
            \begin{equation}
                \Pr_{\bS \sim \mu^{k n_0}} \left[ 
                \Pr_{(\bm{x}, \bm{y}) \sim \mu} \left[
                        \bm{y} \inp{\bm{x}}{\bw_i} < \frac{\gamma}{2}    
                    \right] \leq \frac{\epsilon}{k} \text{ for all $i\in [k]$} \right] \geq 1 - \frac{\delta}{4},
            \end{equation}
            and applying the union bound again,
            \begin{equation}\label{eq:convex_hull_all_accurate}
                \Pr_{\bS \sim \mu^{k n_0}} \left[
                    \Pr_{(\bm{x}, \bm{y}) \sim \mu} \left[
                        \min_{i \in [k]} \bm{y}  \inp{\bm{x}}{\bw_i} < \frac{\gamma}{2} 
                    \right] \leq \epsilon 
                \right] \geq 1 - \frac{\delta}{4}.
            \end{equation}
            Finally, if $(x, y) \in \bbS^{d-1} \times \Set{\pm 1}$ satisfies $y  \inp{x}{\bw_i} \geq \gamma/2$ for all $i \in [k]$, then noting that $\norm{\bw} \le 1$, we have 
            \begin{equation}
            \label{eq:large_margin_z}
                y  \inp{x}{\bz}
                 \geq y \inp{x}{\bw}
                = y  \inp{x}{\frac{1}{k} \sum_{i=1}^k \bw_i} \ge  \gamma/2.
            \end{equation}
 Thus, from \eqref{eq:convex_hull_all_accurate}, we have  
   \[ \Pr_{\bS \sim \mu^{k n_0}}\left[\norm{\bw}_2 <\frac{\gamma}{2}\right]\le \frac{\delta}{4}\] 
and
            \begin{equation}\label{eq:likely_large_margin}
                \Pr_{\bS \sim \mu^{kn_0}} \left[
                    \Pr_{(\bm{x}, \bm{y}) \sim \mu} \left[
                        \bm{y} \inp{\bm{x}}{\bz} \geq \frac{\gamma}{2}
                    \right] \geq 1 - \epsilon 
                \right] \geq 1 - \frac{\delta}{4}.
            \end{equation}

            By applying \Cref{lemma:good_rounding_scheme} with $\alpha \coloneqq \gamma/2$, after rounding $\bz$ to $\tilde{\bz}$, we have   
            $\norm{\tilde{\bz} - \bz}_2 < \frac{\gamma}{2}$.
            Thus if $(x,y) \in \bbS^{d-1} \times \set{\pm 1}$ satisfy $y \inp{x}{\bz} \geq \gamma/2$,  then
            \[
                y \inp{x}{\tilde{\bz}} = y \inp{x}{\bz} + y \inp{x}{\tilde{\bz} - \bz }
                \geq \frac{\gamma}{2} - \norm{\tilde{\bz} - \bz}_2
                > 0,
            \]
            namely $\overline{c}_{\tilde{\bz}}(x)=y$. Thus,
            \[
                \Pr_{\bS \sim \mu^{kn_0}} \left[
                    \loss_\mu (\overline{c}_{\tilde{\bz}}) \leq \epsilon
                \right] \geq 1 - \frac{\delta}{4},
            \]
            which completes the proof of the claim.
        \end{proof}

        We now complete the proof by addressing list replicability. Let $\beta$ be as in \Cref{lemma:good_rounding_scheme}.
        Applying \cref{thm_vec_bernstein}, since $\bw$ is the average of $k$ i.i.d. random variables in $\bbS^{d-1}$, there exists $k=k(\gamma,d) \in \N$ such that 
        \[
            \Pr_{\bS \sim \mu^{k n_0}}\left[
                \norm{\bw - \Ex[\bw]}_2 \geq \frac{\gamma \beta }{2}
            \right] \leq \frac{\delta}{4}.
        \]
        Since $\bm{z}=\frac{\bw}{\norm{\bw}_2}$, by applying the union bound to \eqref{eq:large_w} and the above inequality, we have 
        \[
            \Pr_{\bS \sim \mu^{k n_0}}\left[
                \norm{\bm{z} - \Ex[\bm{z}]}_2 \geq \beta
            \right] \leq \frac{\delta}{2}.
        \]        
      
     Consequently, by \cref{lemma:good_rounding_scheme}, with probability at least $1 - \delta / 2$, the rounding scheme ($\mathrm{round}_{\frac{\gamma}{2}}$) outputs one of at most $d$ hypotheses.
        Applying a union bound with \Cref{claim:Average_PAC} completes the proof of the upper bound of \Cref{thm:main}.
    \end{proof}

\section{Disambiguations of gap Hamming distance}\label{sec:GHD} 
We prove \Cref{thm:GHD} in this section. 
\paragraph{The Littlestone dimension of disambiguations.} The key to obtaining \eqref{eq:ghd-ldim} is to use an embedding of bounded margin half-spaces in dimension $d$ into the Boolean cube, which allows us to disambiguate $\cH^d_\gamma$ using a disambiguation of the Gap Hamming Distance problem in dimension $O(d)$. The existence of such an embedding was proved in \cite{hatami2023borsuk}, which we rephrase as follows.

\begin{lemma}[Adapted from {\cite[Lemma 3.2]{hatami2023borsuk}}]
        \label{lem_inner_prod_preserving_map}
            Let $\gamma \in (0,1)$ and $n\in \N$. There exist $d = \Omega\left((1-\gamma)^2\cdot n/\log(1/(1-\gamma))\right)$,  $\gamma' \in (0,1)$ and a map $\xi \colon \bbS^{d-1} \to \Set{\pm1}^n$ such that for all $u,v \in \bbS^{d-1}$, we have
            \[
                \inp{u}{v} \geq \gamma'
                \implies
                \inp{\xi(u)}{\xi(v)} \geq \gamma n
            \]
            \[
                \inp{u}{v} \leq - \gamma'
                \implies
                \inp{\xi(u)}{\xi(v)} \leq -\gamma n
            \]
\end{lemma}
 
Fix $\gamma \in (0,1)$. Let $\{M_n\}_{n=1}^{\infty}$ be a family of total functions  which disambiguates $\{\operatorname{GHD}_\gamma^n\}_{n=1}^\infty$, and let $d = d(n)$ and $\gamma'$ be as provided by \cref{lem_inner_prod_preserving_map}.

We will use this lemma along with the functions $\Set{M_n}_{n=1}^{\infty}$ to disambiguate the family of partial concept classes $\{ \cH_{\gamma'}^{d(n)} \}_{n=1}^{\infty}$. To this end, we disambiguate each partial concept $c_w\in \cH_{\gamma'}^d$ (defined in \eqref{eq:def_large_margin}) to
\[
\overline{c}_{w}(x)
\coloneqq
M_n \big( \xi(w), \xi(x) \big).
\]
Let us verify that $\overline{c}_{w}$ is, in fact, a disambiguation of $c_{w}$.
 
Suppose that $c_w(x) = 1$. By definition, this occurs exactly when $\inp{w}{x} \geq \gamma'$. It follows from the properties of $\xi$ that
\[
    \inp{\xi(w)}{\xi(x)} \geq \gamma n.
\]
Therefore, for such $w,x$, we have 
\[
    \overline{c}_w(x)=M_n(\xi(w),\xi(x))= \GHD_\gamma^n(\xi(w),\xi(x))= 1 = c_w(x).
\]
A similar argument shows that if $c_w(x) = -1$, then $\overline{c}_{w}(x) = -1$. We deduce that $\overline{c}_{w}$ indeed disambiguates $c_w$, and $\overline{\cH}^d_{\gamma'}$ is a disambiguation of $\cH^d_{\gamma'}$. 
    
Finally, note that by construction, any shattered mistake tree in $\overline{\cH}_{\gamma'}^d$ corresponds to a shattered mistake tree of the same depth in $M_n$. Therefore, $\ldim(\overline{\cH}_{\gamma'}^d) \leq \ldim(M_n)$.  This combined with \Cref{thm:disambig} implies that
\[
    \ldim(M_n)
    \geq \ldim \left( \overline{\cH}_{\gamma'}^d \right)
    =\Omega \left( \sqrt{\log d(n)} \right)
    = \Omega \left( \sqrt{\log n} \right).\qedhere
\]

\paragraph{Communication complexity of disambiguations.} To complete the proof of \cref{thm:GHD} we use the known relationships between Littlestone dimension, \emph{margin}, \emph{distributional discrepancy}, and public-coin randomized communication complexity. Given a matrix $M\in \{\pm 1\}^{\cX\times \cY}$, the margin of $M$ is defined 
\[
\margin(M)\coloneqq\max_{\substack{d\in \N, \\ u_x,u_y \in \bbS^d}}\min_{(x,y)} M(x,y)\cdot \inp{u_x}{u_y}.
\] 
In other words, $\margin(M)$ is the largest $\gamma$ such that $M$ appears as a submatrix of $\cH^d_{\gamma}$ for some $d$. 

Let $\{M_n\}_{n=1}^\infty$ be a family of disambiguation of $\{\mathrm{GHD}_\gamma^n\}_{n=1}^\infty$. By \eqref{eq:ghd-ldim}, we know that $\ldim(M_n)=\Omega(\sqrt{\log n})$. It thus follows from \eqref{eq:Ldim_of_margin} that
\[
\margin(M_n)=O\left(\frac{1}{\sqrt[4]{\log n}}\right).
\]
Finally, invoking the equivalence of margin and discrepancy by~\cite{linialshraibman} and the relation between discrepancy and  randomized communication complexity by~\cite{chor1988unbiased} (see also~\cite[Proposition 3.3 ]{MR4494342}) shows that the public-coin randomized communication complexity of $M_n$ is 
\[\Omega(\log(\margin(M_n)^{-1}))= \Omega(\log\log n).\qedhere \] 

\section*{Acknowledgments}

We are grateful to Arkadev Chattopadhyay for pointing out the connection to pseudo-determinism and for valuable comments on the exposition. We also thank Zachary Chase for clarifying the implicit bound in~\cite{ghazi2021sample}.

\bibliographystyle{amsalpha}  
\bibliography{refs} 
\appendix

\section{Replicability and privacy notions}
In this section, we state the formal definitions of shared-randomness replicability, DP-learnability, and global stability. 
\subsection{Shared-randomness replicability}

Let $\cA(S,r)$ be a randomized learning rule, where $r$ denotes the random seed.

\begin{definition}[Shared-randomness replicability~\cite{ghazi2021user,ILPS22}]
\label{def:shared}
A concept class $\cC \subseteq \Set{\pm,\star}^\cX$ is \emph{shared-randomness replicable} if there exists a learning rule $\cA$ and a sample complexity function $n(\epsilon, \delta)$ such that, for every $\epsilon, \delta > 0$ and every realizable distribution $\mu$, the following conditions hold:
\begin{itemize}
\item Small population loss: 
$ \Pr_{\bm{S}\sim \mu^n, \bm{r}}[\loss_{\mu}\left(\cA(\bm{S},\bm{r})\right)>\epsilon] \le \delta. $
\item Replicability with shared randomness:
$
\Pr_{\bm{S,S'}\sim \mu^n, \bm{r}}[\cA(\bm{S},\bm{r}) = \cA(\bm{S'},\bm{r})] \ge 1-\delta. 
$
\end{itemize}
\end{definition}

 One could consider shared-randomness replicability to be a weak form of replicability, as different executions of the algorithm can use the same random seed.

\subsection{Differential privacy}

The widely adopted approach for ensuring privacy in machine learning is the differential privacy (DP) framework, introduced in~\cite{10.1007/11681878_14}.   Informally, differential privacy in learning means that no single labeled example in the input dataset significantly impacts the learner's output hypothesis.  In other words, the output distribution of a differentially private randomized learning algorithm remains nearly unchanged if a single data point is modified.

Differential privacy is quantified with two parameters $\epsilon,\delta>0$. We say that two probability distributions $p$ and $q$ are $(\epsilon,\delta)$-indistinguishable, if for every event $E$, we have
\[p(E) \le e^\epsilon q(E)+\delta \ \text{ and }\  q(E) \le e^\epsilon p(E)+\delta.\]
Two random variables are $(\epsilon,\delta)$-indistinguishable if their distributions satisfy this condition.

\begin{definition}[Differential privacy]
Given $\epsilon,\delta>0$, 
a randomized learning rule \[\bm{\cA}:(\cX \times \Set{\pm 1})^n \to \Set{\pm}^\cX\] is $(\epsilon,\delta)$-differentially-private if for every two samples $S, S' \in (\cX \times \Set{\pm})^n$ differing on a single example, the random variables $\bm{\cA}(S)$ and $\bm{\cA}(S')$ are $(\epsilon,\delta)$-indistinguishable. 
\end{definition}

We emphasize that $(\epsilon,\delta)$-indistinguishability must hold for every such pair of samples, regardless of whether they are drawn from a (realizable) distribution.

The special case where $\delta=0$ is known as \emph{pure differential privacy}, while the more general case where 
$\delta>0$ is referred to as \emph{approximate differential privacy}.

In approximate differential privacy, the parameters $\epsilon$ and $\delta$ are typically set as follows: $\epsilon$ is taken to be a small fixed constant (e.g., 
0.1), while $\delta$ is a negligible function, $\delta=n^{-\omega(1)}$.

\begin{definition}[Approximate differentially drivate learnability]
\label{def:approx_DP}
We say that a concept class $\cC \subseteq \Set{\pm 1, \star}^\cX$ is \emph{approximate differentially private learnable} (\emph{DP-learnable}) if there is a learning rule $\bm{\cA}: (\cX \times \Set{\pm 1})^* \to \Set{\pm 1}^\cX$ with sample complexity $n(\epsilon,\delta)$ such that for every $\epsilon,\delta>0$ the following holds. 
\begin{itemize}
\item The class $\cC$ is $(\epsilon,\delta)$-PAC learnable by $\bm{\cA}$ using $n(\epsilon,\delta)$ samples.

\item The learning rule $\bm{\cA}$ applied to samples of size $n(\epsilon,\delta)$ is $(\epsilon'(n),\delta'(n))$-differentially private learnable where $\epsilon'(n)\leq 0.1$ and $\delta'(n)\leq n^{-w(1)}$. 
\end{itemize}

\end{definition}

\begin{definition}[Pure Differentially Private Learnability] 
\label{def:pure_DP}
We say that a concept class $\cC \subseteq \Set{\pm 1, \star}^\cX$  is \emph{pure differentially private learnable} (\emph{pure DP-learnable}) if $\cC$ is PAC learnable by a $(0.1,0)$-differentially private learning rule. 
\end{definition}
\subsection{Global stability}\label{sec:global_stability}
The concept of replicability in PAC learning first emerged in \cite{BLM20, Alon_22_private_and_online} in the study of differential privacy of PAC learning algorithms. These works introduced a notion of replicability known as \emph{global stability} to derive privacy guarantees from online learnability.

\begin{definition}
 A learning rule $\bm{\cA}$ for a concept class $\cC \subseteq \Set{\pm 1,\star}^\cX$ is $(\epsilon,\rho)$-globally stable if for every realizable distribution $\mu$, there is a hypothesis $h \in \Set{\pm 1}^\cX$ with population loss $\loss_\mu(h) \le \epsilon$ satisfying
\[ \Pr_{\bm{S} \sim \mu^n} [\bm{\cA}(\bm{S})  =h] \geq \rho, \ \ \text{ where } n=n(\epsilon).  \] 
 We define $\gs_\epsilon(\cC)$ to be the supremum of $\rho$ such that there is a $(\epsilon,\rho)$-globally stable learner for $\cC$. The \emph{global stability parameter} of $\cC$ is then defined as \[\gs(\cC)\coloneqq \inf_{\epsilon>0} \gs_\epsilon(\cC).\]
\end{definition}

The definition of global stability might initially seem weak, as a globally stable learner is not necessarily a PAC learner. In particular, since $\rho$ can be a small constant, there may be a probability as great as $1 - \rho$ that the learning rule outputs a hypothesis with large population loss. However, as discussed in the next section, global stability is equivalent to the seemingly stronger notion of list replicability.

\section{Equivalence of global stability and list replicability}\label{sec:List_eq_global}

In~\cite{chase2023replicabilitystabilitylearning}, Chase, Moran and Yehudayoff proved that for every total class $\cC \subseteq \Set{\pm 1}^\cX$, list replicability is equivalent to global stability. It is easy to check that their proof applies to partial concept classes, resulting in the following relationship between the list replicability number and the global stability parameter.
 
\begin{theorem}
\label{thm:List_equiv_global}
    Let $\cC$ be any total or partial concept class on the domain $\cX$. Then for every $\epsilon\in (0,1)$,
\[
\gs_\epsilon(\cC)\geq \frac{1}{\LR_\epsilon(\cC)}\qquad  \text{ and } \qquad
\LR_\epsilon(\cC) \leq \frac{1}{\gs_{\epsilon/3}(\cC)}.
\]
Consequently, $\gs(\cC) = \frac{1}{\LR(\cC)}$. 
\end{theorem} 

\begin{proof}
We first prove that $\gs_\epsilon(\cC)\geq \frac{1}{\LR_\epsilon(\cC)}$. Let $\epsilon>0$ be an accuracy parameter, and let $\bm{\cA}$ be an $(\epsilon, L)$-list replicable learner for $\cC$ with sample complexity $n=n(\epsilon,\delta)$. Let $\mu$ be any realizable distribution on $\cX \times \{\pm 1\}$, and let $h_1,\ldots, h_{L}$ be the list of hypotheses guaranteed by \Cref{def:list}.

By the pigeonhole principle, at least one $h_i$ satisfies
    \[
        \Pr_{\bS \sim \mu^n}[\bm{\cA}(\bm S) = h_i] 
        \ge \frac{1-\delta}{L}.
    \]
Since this statement holds for arbitrary $\delta>0$, $\bm{\cA}$ is itself an $(\epsilon, \rho)$-globally stable learner for all $\rho<\frac{1}{L}$. We may conclude that $\gs_\epsilon(\cC)\geq \frac{1}{\LR_\epsilon(\cC)}$.

Next, we prove that $\LR_\epsilon(\cC) \leq 1/\gs_{\epsilon/3}(\cC)$. Let $\epsilon>0$ be an accuracy parameter, and let $\bm{\cA}$ be an $(\epsilon/3, \rho)$-globally stable learner for $\cC$ with sample complexity $n_0 = n_0(\epsilon)$. By the stability assumption, for every realizable distribution $\mu$ on $\cX \times \{\pm 1\}$, there exists $h^*:\cX \to \Set{\pm 1}$ satisfying
    \begin{equation}\label{eq_stable_hypothesis}
        \loss_\mu(h^*) \le \frac{\epsilon}{3} \ \text{ and }\ \Pr_{\substack{\bm S\sim \mu^{n_0} }}[\bm{\cA}(\bm S)=h^*] \geq \rho.
    \end{equation}
For every $h \in \{\pm 1\}^\cX$ and realizable distribution $\mu$, define
    \[
        p(h) \coloneqq \Pr_{\bS \sim \mu^{n_0}}[\bm{\cA}(\bm S) = h], 
    \]
Denote $L \coloneqq \left\lfloor \frac{1}{\rho}\right\rfloor$, so that $\rho \in \left(\frac{1}{L +1}, \frac{1}{L }\right]$, and 
let $\alpha \coloneqq \rho - \frac{1}{L+1} > 0$. Define the list $\Lambda$ of good and likely hypotheses
    \[
        \Lambda \coloneqq
        \Set{h \in \Set{\pm 1}^\cX
        \ | \
        p(h) > \frac{1}{L+1}
        \text{ and }
        \loss_{\mu}(h) \le \epsilon}. 
    \]
Note that $|\Lambda| \leq L$ and $\Lambda$ is nonempty, as it contains $h^*$. Therefore, to construct an $(\epsilon,L)$-list replicable learner, it suffices to show that for any confidence parameter $\delta > 0 $, the learning rule outputs a hypothesis from $\Lambda$ with probability at least $1-\delta$.
    
Let $t\coloneqq t(\alpha,\delta)$ and $n_1 \coloneqq n_1(\epsilon,t)$ be sufficiently large integers to be determined later. We propose the following learning rule $\bm{\cA}'$ with sample complexity $t n_0+n_1$.
  
\begin{algorithm}[h]
\caption{The learning rule $\bm{\cA}'$}
\label{alg:my-alg}
\begin{algorithmic}[1]
    \STATE Sample a dataset:
    \[
      \bm{S}=(\bP, \bm{Q}) \sim \mu^{t n_0 + n_1},
      \quad \text{where } \bP=(\bP_1,\ldots, \bP_t) \sim (\mu^{n_0})^t = \mu^{t n_0},
      \quad \text{and} \quad \bm{Q} \sim \mu^{n_1}.
    \]
    \STATE Define the empirical estimate of \(p(h)\) as
    \[
        \operatorname{freq}_\bP(h)\coloneqq \frac{|\Set{i\in [t] \ | \ \bm{\cA}(\bP_i)=h}|}{t}.
    \]

    \STATE \textbf{Output} any hypothesis \( h \in \{\pm 1\}^{\cX} \) satisfying:
    \begin{itemize}
      \item \(\operatorname{freq}_\bP(h)\ge \rho - \tfrac{\alpha}{2}\)
      \item \(\loss_{{\bm Q}}(h) \le \tfrac{2\epsilon}{3}\)
    \end{itemize}
    If no such \(h\) exists, output an arbitrary \(h\) corresponding to “failure.”

\end{algorithmic}
\end{algorithm}

Denote by $\cY$ the set of all $h$ with $\operatorname{freq}_\bP(h) > 0$ in \Cref{alg:my-alg}, and note that $|\cY| \leq t$.
        % \[
        %     \cY \coloneqq \{h \in \{\pm 1\}^\cX \ |\ \operatorname{freq}_\bP(h) > 0\}.
        % \]
%
To show that $\bm{\cA}'$ outputs a hypothesis from $\Lambda$ with probability at least $1-\delta$, we will condition on the events
    \begin{align*}
        A : \quad &|\loss_\mu(h) - \loss_{\bm Q}(h)| \leq \frac{\epsilon}{3} \text{ for all } h \in \cY\\
        B : \quad &|p(h) - \operatorname{freq}_\bP(h)| < \frac{\alpha}{2} \text{ for all } h \in \Set{\pm1}^\cX
    \end{align*}
To guarantee that both events are likely, we prove the following claim.
    
    \begin{claim}\label{claim_indicator_convergence}
        There exist integers $t(\alpha,\delta)$ and $n_1(\epsilon,t)$ such that
        \[
            \Pr_{\bP \sim \mathcal{D}^{t n_0}} [B] \geq 1 - \frac{\delta}{2}
            \quad
            \text{and}
            \Pr_{{\bm Q} \sim \mu^{n_1}} [A] \geq 1 - \frac{\delta}{2}.
        \]
    \end{claim}

    \begin{proof}[Proof of \cref{claim_indicator_convergence}]
        For the choice of $t$ and the proof of the first inequality, we use the uniform convergence property of the family of indicator functions on $\set{\pm1}^\cX$. More precisely, for $f \in \Set{\pm 1}^\cX$,
        define $\mathbb{I}_f:\Set{\pm 1}^\cX \to \Set{0, 1}$ as 
        \[
            \mathbb{I}_f(f') \coloneqq 
            \begin{cases}
                1 & f'=f \\
                0 & \text{otherwise}
            \end{cases}.
        \]
        The class 
        \[
            \cI \coloneqq \Set{\mathbb{I}_f \ | \ f \in \Set{\pm 1}^\cX}
        \]
        has VC dimension $1$, and therefore, it satisfies the uniform convergence property.
        For $\bP_i \sim \mu^{n_0}$,  $\bm{\cA}(\bP_i)$ induces a probability distribution $\mu$ on $\Set{\pm 1}^\cX$, and we have
        \[
            1-p(h)=\Pr_{\bP_i \sim \mu^{n_0}}[\bm{\cA}(\bP_i) \neq h] = \loss_\mu(\mathbb{I}_h), 
        \] 
        while $1-\operatorname{freq}_\bP(h)$ corresponds to the empirical loss of $\mathbb{I}_h$ on $(\mathbb{I}_{h_1}, \ldots, \mathbb{I}_{h_t}) \sim \mu^t$.
        Thus, by the uniform convergence property on $\cI$, our claim holds.

        Now that we have $t$, we can define $n_1$ and prove the second inequality. Note that for every $h \in \{\pm1\}^\cX$, for ${\bm Q} \sim \mu^{n_1}$, $\loss_{\bm Q}(h)$ is an average of ${n_1}$ samplings of a Bernoulli random variable with expectation $\loss_\mu(h)$. Thus, by Hoeffding's inequality, there exists $n_1 = n_1(\epsilon', t)$ such that
        \begin{equation}\label{eq_single_convergence}
            \Pr_{{\bm Q} \sim \mu^{n_1}} \left[ \left|\loss_\mu(h) - \loss_{\bm Q}(h)\right| > \frac{\epsilon}3 \right] \leq \frac{\delta}{2t}.
        \end{equation}
        Thus, by the union bound, we have  
        \begin{equation}\label{eq_multi_conv}
            \Pr_{{\bm Q} \sim \mu^{n_1}} \left[ \left|\loss_\mu(h) - \loss_{\bm Q}(h)\right| \leq \frac{\epsilon}{3} \text{ for all } h \in \cY \right] \geq 1 - \frac{\delta}{2}.
    \end{equation}
    \end{proof}
    
A direct consequence of Claim \ref{claim_indicator_convergence} is that
    \[
        \Pr_{\bS \sim \mu^{t n_0 + n_1}} \left[
            A, B
        \right] \geq 1 - \delta.
    \]
Condition on events $A$ and $B$, and let $h^*$ be a  stable hypothesis for $\bm{\cA}$, as described in \eqref{eq_stable_hypothesis}. We will show that $h^*$ is a candidate for output, so $\bm{\cA}'$ will not output ``failure''. To check the first condition for output, we combine $B$ and \eqref{eq_stable_hypothesis} to show that
    \[
        \operatorname{freq}_\bP(h^*) \geq p(h^*) - \frac{\alpha}{2} \geq \rho - \frac{\alpha}{2}.
    \]
Moreover, $\rho - \frac{\alpha}{2} > 0$, so $h^* \in \cY$. We may therefore apply $A$ to show that $h^*$ satisfies the second condition for output,
    \[
        \loss_{\bm Q}(h^*)
        \leq \loss_\mu(h^*) + \frac{\epsilon}{3}
        \leq \frac{2\epsilon}{3}.
    \]

Finally, let $h_o$ be any output of $\bm{\cA}'$, conditioned on $A$ and $B$. Then, $h_o$ satisfies the condition $\operatorname{freq}_\bP(h_o) \geq \rho - \frac{\alpha}{2}$, so because of $B$,
    \[
        p(h_o) > \operatorname{freq}_\bP(h_o) - \frac{\alpha}{2} \geq \rho - \alpha = \frac{1}{L+1}.
    \]
Furthermore, $h_o$ also satisfies the condition $\loss_{\bm Q}(h_o) \leq \frac{2\epsilon}{3}$, so because of $A$,
    \[
        \loss_\mu(h_o)
        \leq \loss_{\bm Q}(h_o) + \frac{\epsilon}{3}
        \leq \epsilon.
    \]
    Thus, $h_o$ must be in $\Lambda$.
\end{proof}

\end{document}